\pdfoutput=1
\documentclass{article} % For LaTeX2e
\usepackage{graphicx} % more modern
\usepackage{subfigure} 
\usepackage{algorithm}
\usepackage{algorithmic}
\usepackage{color}

\usepackage{amsmath,amssymb,amsfonts,amsthm}

\usepackage{url}

\newtheorem{theorem}{Theorem}

\newtheorem{lemma}{Lemma}
\newtheorem{lemma-ap}{Lemma}

\newcommand{\parens}[1]{\left(#1\right)}

\newcommand{\expectsub}[2]{\mathrm{E}_{#1}\left[{#2}\right]}
\newcommand{\kl}{{\cal D}}
\newcommand{\ignore}[1]{}
\DeclareMathOperator*{\argmax}{argmax}
\DeclareMathOperator*{\argmin}{argmin}

\newcommand{\lmax}{L_{\max}}

\begin{document}
\title{A PAC-Bayesian Tutorial with A Dropout Bound}
\author{David McAllester}
\date{July 5, 2013}

\maketitle

\abstract{This tutorial gives a concise overview of existing PAC-Bayesian theory focusing on three generalization bounds.
The first is an Occam bound which handles rules with finite precision parameters and which states that generalization loss
is near training loss when the number of bits needed to write the rule is small compared to the sample size.
The second is a PAC-Bayesian bound providing a generalization guarantee for posterior distributions rather than for individual rules.
The PAC-Bayesian bound naturally handles infinite precision rule parameters, $L_2$
regularization, {\em provides a bound for dropout training}, and defines a natural notion of a single distinguished PAC-Bayesian posterior distribution.
The third bound is a training-variance bound ---
a kind of bias-variance analysis but with bias replaced by expected training loss.  The training-variance bound
dominates the other bounds but is more difficult to interpret.  It seems to suggest variance reduction methods such as bagging and may ultimately provide
a more meaningful analysis of dropouts.}

\section{Introduction}

PAC-Bayesian theory blends Bayesian and frequentist approaches to the theory of machine learning.
PAC-Bayesians theory assumes a probability distribution on ``situations'' occurring in nature and a
prior weighting on ``rules'' expressing a learners preference for some rules over others.  There is no assumed
relationship between the learner's bias on rules and nature's distribution on situations.  This is different from
Bayesian inference where the starting point is a (perhaps subjective) {\em joint} distribution on rules and situations
inducing a conditional distribution on rules given situations.  The acronym PAC stands for Probably Approximately Correct
and is borrowed from Valiant's notion of PAC learnability \cite{PAC}.
PAC-Bayesian generalization bounds \cite{McAllester99,seeger2003,maurer2004,langford2006,catoni2007pac,germain2009pac}
govern the performance (loss) when stochastically selecting rules from a ``posterior'' distribution.  The performance guarantee
involves the learner's bias and an (unrelated) sample of situations.

This tutorial provides a concise overview
of existing PAC-Bayesian theory focusing on three bounds.  The first is an Occam bound.  An Occam bound assumes a discrete (countable) set of rules and bounds the loss of an individual rule.
The Occam bound immediately yields guarantees for rules with sparse finite precision parameters.
The second is a PAC-Bayesian bound governing the loss of a stochastic process which draws rules from a PAC-Bayesian rule posterior.
The PAC-Bayesian bound easily handles $L_2$ regularization
of infinite-precision parameters producing bounds closely related to support vector machines.
It also provides bounds for a form of dropout learning \cite{deng2013new}.

The third bound is a training-variance bound similar to a bias-variance analysis but with bias replaced by expected training loss.
This bound assumes a given learning algorithm
and provides an upper bound
on the expected generalization loss in terms of the expected training loss and a measure of the variance of the output of the learning algorithm.
While the training-variance bound is clearly tighter than the PAC-Bayesian bound,
the training-variance bound is difficult to interpret.  The training-variance bounds seems to suggest variance-reduction methods such as
bagging~\cite{breiman1996bagging}.

Unbounded loss functions, such as square loss or log loss, can lead to unstable learning algorithms.
Learning algorithms that minimize training loss for an unbounded loss function
tend to be overly sensitive to outliers --- training points of very high loss.
Learning algorithms based on bounded loss functions tend to be more robust (stable).
An unbounded loss function can be converted to a bounded loss by selecting an ``outlier threshold'' $\lmax$
and replacing the unbounded loss $L$ by $\min(L,\lmax)$.

PAC-Bayesian bounds apply only to bounded loss functions.  We assume the loss is bounded to the interval $[0,\lmax]$.
It is of course possible to rescale any bounded loss function into the interval $[0,1]$.  However,
this obscures the significance of the choice of  $\lmax$ in the design of a robust versions of
square loss or log loss.  For this reason we leave $\lmax$ explicit in the statement of the bounds.

This tutorial also discusses two improvements or clarifications of the three bounds mentioned above.
The first applies the training-variance bound to the learning algorithm defined by the PAC-Bayesian posterior.
Unfortunately the results suffer from looseness in the analysis and remain difficult to interpret.
The second tightens the Occam bound by incorporating the loss variance into the bound.  We show that the improvements
achievable in this way a fundamentally limited.

\section{An Occam Bound}
\label{sec:Occam}

Let ${\cal H}$ be a set of ``rules'', ${\cal S}$ be a set of ``situations'', $\lmax > 0$ be a real number,
and $L$ be a loss function such that for a rule $h \in {\cal H}$ and a situation $s \in {\cal S}$ we have that $L(h,s) \in [0,\lmax]$.
We let $D$ be a probability distribution (measure) on ${\cal S}$ and let $P$ be a distribution (measure) on ${\cal H}$.
We think of $D$ as a distribution on situations occurring in nature and $P$ as learner bias on rules.
There is no assumed relationship between $D$ and $P$.\footnote{We should assume that the loss function $L$ is
measurable with respect to $D$ and $P$. Here we will avoid this level of rigor.}
We are interested in drawing a sequence $S$ of $N$ situations IID from $D$
($S \sim D^N$) and then selecting $h$ based on $S$ so as to minimize the ``generalization loss'' $L(h) = \expectsub{s \sim D}{L(h,s)}$.
When the sample $S$ is clear from context we will write $\hat{L}(h)$ for $\frac{1}{N} \sum_{s \in S} L(h,s)$.

For Occam bounds we consider the case where ${\cal H}$
is discrete (countable).  An Occam bound states that with probability at least $1-\delta$ over the draw of the sample $S \sim D^N$
we have $L(h) \leq B(P(h),S,\delta)$ simultaneously for all $h$ where $B(P(h),S,\delta)$ is different in different bounds.
While various Occam bounds have appeared in the literature, here we will consider only the following.

\begin{theorem} With probability at least $1-\delta$ over the draw of $S \sim D^N$ we have that the following holds simultaneously for all $h$.

\begin{equation}
\label{Occam}
\mathbf{L}(h) \leq \inf_{\lambda > \frac{1}{2}}\;\frac{1}{1-\frac{1}{2\lambda}}\parens{\widehat{L}(h) + \frac{\lambda\lmax}{N}\parens{\ln\frac{1}{P(h)} + \ln\frac{1}{\delta}}}
\end{equation}
\end{theorem}

\begin{proof} We consider the case of $\lmax = 1$, the case for general $\lmax$ follows by rescaling the loss function.
Define $\epsilon(h)$ by
$$\epsilon(h) = \sqrt{\frac{2L(h)\parens{\ln\frac{1}{P(h)} + \ln\frac{1}{\delta}}}{N}}.$$
For a given $h \in {\cal H}$ the relative Chernoff bound \cite{angluin1977fast}
states that
$$P_{S \sim D^N}\parens{\hat{L}(h) \leq L(h) - \epsilon(h)} \leq e^{-N\frac{\epsilon(h)^2}{2L(h)}} = \delta P(h).$$
Hence, for a fixed $h$ the probability of
$L(h) > \hat{L}(h) + \epsilon(h)$ is at most $P(h)\delta$.
By the union bound the probability that there exists an $h$
with $L(h) > \hat{L}(h) + \epsilon(h)$ is at most the sum over $h$ of $P(h)\delta$ which equals $\delta$.
Hence we get that with probability at least $1-\delta$ over the draw of the sample the following holds simultaneously for all $h$.
$$L(h) \leq \widehat{L}(h) + \sqrt{L(h)\parens{\frac{2\parens{ln\frac{1}{P(h)} + \ln\frac{1}{\delta}}}{N}}}$$
Using
$$\sqrt{ab} = \inf_{\lambda > 0}\;\frac{a}{2\lambda} + \frac{\lambda b}{2}$$
we get
$$L(h) \leq \widehat{L}(h) + \frac{L(h)}{2\lambda} + \frac{\lambda\parens{ln\frac{1}{P(h)} + \ln\frac{1}{\delta}}}{N}$$
Solving for $L(h)$ yields the result.
\end{proof}

An important observation for (\ref{Occam}) is that there is no point in taking $\lambda$ to be large.  Restricting $\lambda$ to be less than $\lambda_{\max}$
increases the bound by a factor of at most $1/(1-1/2\lambda_{\max})$.  For example, restricting $\lambda$ to be less than 10
increases the bound by a factor of at most $20/19$.  So for practical purposes we can assume that $\lambda$ is no larger than 10.

\subsection{Finite Precision Bounds}
\label{sec:floating}

As an example application of (\ref{Occam}) we can consider rules of the form $h_\Theta$ for some parameter vector $\Theta \in \mathbb{R}^d$
where each component of the vector $\Theta$ is represented with a $b$-bit finite precision representation.  In this case
we can take the prior $P$ to be uniform on the $2^{bd}$ possible rules and (\ref{Occam}) then gives that
with probability at least $1-\delta$ over the draw of the sample we have that the following holds
simultaneously for all such $\Theta$.
$$L(h_\Theta) \leq \inf_{\lambda > \frac{1}{2}} \; \frac{1}{1-\frac{1}{2\lambda}}\parens{\hat{L}(h_\Theta) + \frac{\lambda\lmax}{N}\parens{(\ln 2)bd + \ln \frac{1}{\delta}}}$$

We can also consider sparse representations.  For $\Theta \in \mathbb{R}^d$ we say that $\Theta$ has sparsity level $s$
if at most $s$ components of $\Theta$ are non-zero.  We can then represent a sparse vector by first specifying the sparsity $s$ and then listing $s$ pairs
each of which specifies a non-zero component and its value.  Intuitively
we can write a rule by first using $\log_2 d$ bits to specify $s$ plus $(\log_2 d)b$
bits for each pair of a component index and $b$-bit parameter value representation.  The probability of a rule $h$ can always be taken to be $2^{-|h|}$ where $h$ is the number of bits
needed to name $h$.  Formally we avoid coding and instead defining a probability distribution where we first select $s$ uniformly from $1$ to $d$
and then select $s$ pairs with indices drawn uniformly form $1$ to $d$ and a parameter representation drawn uniformly from all $2^b$ bit strings.  In this case
we get that with probability at least $1-\delta$ over the draw of the sample of situations we have the following holds simultaneously for all sparsity levels $s$ and
$\Theta$ with sparsity $s$ and with $b$-bit representations for the non-zero components of $\Theta$.
$$L(h_\Theta) \leq \inf_{\lambda > \frac{1}{2}} \; \frac{1}{1-\frac{1}{2\lambda}}\parens{\hat{L}(h_\Theta) + \frac{\lambda\lmax}{N}\parens{\ln d + s(\ln d + (\ln 2)b) + \ln \frac{1}{\delta}}}$$
More sophisticated codings of classifiers are possible.  For example, variable precision codes for real numbers can be useful when the error rate is insensitive to the precision.
However, the PAC-Bayesian theorem stated in section~\ref{sec:PAC-Bayes} handles infinite precision parameters and seems generally preferable.  The Occam bound
is included here primarily because of its conceptual simplicity and the intuitive value of its proof.

\section{A PAC-Bayesian Bound}

\label{sec:PAC-Bayes}
Let ${\cal H}$, ${\cal S}$, $\lmax$, $L$, $D$ and $P$ be defined as in section~\ref{sec:Occam}.  We now allow the rule set ${\cal H}$
to be continuous (uncountable).  Let $Q$ be a variable ranging over distributions (measures) on the rule space ${\cal H}$.
For $s \in {\cal S}$ we define the loss $L(Q,s)$ to be $\expectsub{h \sim Q}{L(h,s)}$. We have that $L(Q,s)$ is the loss of
a stochastic process that selects the hypothesis $h$ according to distribution $Q$.
We define $L(Q)$ to be $\expectsub{s \sim D}{L(Q,s)}$.  Given a sample $S = \{s_1,\ldots,s_N\}$
we define $\hat{L}(Q)$ to be $\frac{1}{N} \sum_{s\in S} {L(Q,s)}$.  Finally we will write $\kl(Q,P)$ for the Kullback-Leibler divergence from $Q$ to $P$.
$$\kl(Q,P) = \expectsub{h \sim Q}{\ln\frac{Q(h)}{P(h)}}$$

A PAC-Bayesian theorem uniformly bounds $L(Q)$ in terms of $\hat{L}(Q)$ and $\kl(Q,P)$.
The first PAC-Bayesian theorem was given in \cite{McAllester99}.  Tighter PAC-Bayesian theorems have been given by various authors
\cite{seeger2003,maurer2004,langford2006,catoni2007pac,germain2009pac}.
Here we will focus on the following PAC-Bayesian version of the Occam bound (\ref{Occam}) which can be derived as a corollary of statements by Catoni
\cite{catoni2007pac}.  A proof is included here in appendix~\ref{sec:PAC-Bayes-Proof}. 

\begin{theorem}
\label{thm:PAC-Bayes}
For $\lambda > \frac{1}{2}$ selected before the draw of the sample (for any fixed $\lambda > 1/2$) we have that, with probability
at least $1-\delta$ over the draw of the sample, the following holds simultaneously for all distributions $Q$ on ${\cal H}$.
\begin{equation}
\label{PAC-Bayes}
L(Q) \leq \frac{1}{1-\frac{1}{2\lambda}}\parens{\hat{L}(Q) + \frac{\lambda \lmax}{N}\parens{\kl(Q,P) + \ln \frac{1}{\delta}}}
\end{equation}
\end{theorem}

\medskip
As with (\ref{Occam}), there is no point in taking $\lambda$ in (\ref{PAC-Bayes}) to be large ---
we can in practice assume that $\lambda$ is no larger than 10.  Second, although (\ref{PAC-Bayes}) is not uniform in $\lambda$
we can select $k$ different values $\lambda_1$, $\ldots$, $\lambda_k$ (all of which are selected before the draw of the sample) and by a simple union bound
over these values derive that with probability at least $1-\delta$ over the draw of the sample the following holds simultaneously for all $Q$.
\begin{equation}
\label{PAC-Bayes-Uniform}
L(Q) \leq \min_{1 \leq i \leq k} \;\;\frac{1}{1-\frac{1}{2\lambda_i}}\parens{\hat{L}(Q) + \frac{\lambda_i \lmax}{N}\parens{\kl(Q,P) + \ln \frac{k}{\delta}}}
\end{equation}
For minimizing the bound we can assume $1/2 \leq \lambda \leq 10$ and a small number of values of $\lambda$ should suffice.
While it is possible to give a version of this theorem that is uniform over all $\lambda > 1/2$, achieving this uniformity increases
the complexity of the proof.

\subsection{An Infinite Precision $L_2$ Bound}
\label{sec:L2}

As in section~\ref{sec:floating} we consider rules $h_\omega$ with $\omega \in \mathbb{R}^d$.  Here we also assume that
the rule is scale-invariant --- that the rule $h_\omega$ depends only on the direction of the vector $\omega$.
For example linear predictors of the form
$$h_\omega(x) = \argmax_y \omega^\intercal \Phi(x,y)$$
are scale-invariant.
For scale-invariant rules it is natural to consider the uniform distribution over the directions of $\omega$.
This uniform distribution can be formalized as an isotropic unit-variance prior $P = {\cal N}(0,1)^d$
where ${\cal N}(0,1)$ is the zero mean unit-variance Gaussian distribution.  For $\Theta \in \mathbf{R}^d$ we define the distribution
$Q_\Theta$ to be the isotropic unit-variance Gaussian centered on $\Theta$.  Since only the direction of $\omega$ matters,
we should think of $P$ as the uniform distribution over directions and think of $Q_\Theta$ as a non-uniform distribution over directions.
We then have the following.
\begin{eqnarray*}
L(Q_\Theta) & = & \expectsub{\epsilon \sim {\cal N}(0,1)^d}{L(f_{\Theta+\epsilon})} \\
\\
\hat{L}(Q_\Theta) & = & \expectsub{\epsilon \sim {\cal N}(0,1)^d}{\hat{L}(f_{\Theta+\epsilon})} \\
\\
\kl(Q_\Theta,P) & = & \frac{1}{2}||\Theta||^2
\end{eqnarray*}
The PAC-Bayesian bound (\ref{PAC-Bayes}) then gives that with probability at least $1-\delta$ over the draw of the
sample the following holds simultaneously for all $\Theta$.
\begin{equation}
\label{SVB}
L(Q_\Theta) \leq \frac{1}{1-\frac{1}{2\lambda}}\parens{\hat{L}(Q_\Theta) + \frac{\lambda \lmax}{N}\parens{\frac{1}{2}||\Theta||^2 + \ln \frac{1}{\delta}}}
\end{equation}

\subsection{Binary and Multi-Class classification}

As an example we can consider linear binary classification.  In this case we have that each situation is a pair $(x,y)$ with $y \in \{-1,1\}$
and we have
$$h_\omega(x) = \mathrm{sign}(\omega^\intercal \Phi(x))$$
where $\Phi$ is a feature map such that $\Phi(x) \in \mathbb{R}^d$.  We also use 0-1 loss
$$L(h,(x,y)) = 1_{h(x) \not = y}.$$
We then have
$$L(Q_\Theta,\;(x,y)) = P_{\omega \sim Q_\Theta}[h_\omega(x) \not = y] = P_{\epsilon \sim {\cal N}(0,1)}(\epsilon > y \Theta^\intercal \Phi(x)/||\Phi(x)||).$$
In this case we have that (\ref{SVB}) is very similar to the objective defining a support vector machine but where the hinge loss is replaced by a (non-convex)
sigmoidal loss function (the cumulative of a Gaussian).  In practice the rule $h_\Theta$ is used at test time noting that $\Theta$ is the mean of the distribution
$Q_\Theta$.

As another example we can consider expected loss for multi-class classification.
In this case each situation is a pair $(x,y)$ with
$x \in {\cal X}$ and $y \in {\cal Y}$ and where ${\cal Y}$ is small enough to be feasibly enumerated.
We assume a feature map $\Phi$ with $\Phi(x,y) \in \mathbb{R}^d$ and a loss function $\tilde{L}$ with
$\tilde{L}(\hat{y},y) \in [0,\lmax]$. For $\beta > 0$ we then have the following definitions.
\begin{eqnarray*}
h_\omega(x,y) & = & \frac{\omega^{\intercal}\Phi(x,y)}{||\omega||\;||\Phi(x,y)||} \\
\\
P_{\beta,\omega}(\hat{y}|x) & =  & \frac{1}{Z_{\beta,\omega,x}}\;e^{\beta h_\omega(x,\hat{y})} \\
\\
Z_{\beta,\omega,x} & = & \sum_{\hat{y}\in {\cal Y}} e^{\beta h_\omega(x,\hat{y})} \\
\\
L_{\beta}(h_\omega,\;(x,y)) & = & \expectsub{\hat{y} \sim P_{\beta,\omega}(\cdot|x)}{\tilde{L}(\hat{y},y)}
\end{eqnarray*}
This particular formulation has the property that $h_\omega$ is scale invariant (depends only on the direction of $\omega$)
and $\beta$ is a parameter of the loss function.
This formulation also has the property that $L_\beta(h_\omega,\;(x,y))$ is differentiable in $\omega$.
We can then optimize the right hand side of (\ref{SVB}) by stochastic gradient descent using
$$\nabla_\Theta\; \hat{L}(Q_\Theta) = \frac{1}{N}\sum_{i=1}^N\; \expectsub{\epsilon \sim {\cal N}(0,1)^d}{\nabla_\Theta\; L(h_{\Theta + \epsilon},s_i)}.$$

\subsection{Dropouts}

We now present a dropout bound inspired by the recent success of dropout training in deep neural networks \cite{deng2013new}.
This dropout bound is the only original contribution of this tutorial.

For a given dropout rate $\alpha \in [0,1]$ and vector $\Theta \in \mathbb{R}^d$ we can stochastically generate
a vector $w \in \mathbb{R}^d$ by selecting, for each coordinate $w_i$, the value 0 with probability $\alpha$ (dropping the coordinate $\omega_i$)
or with probability $1-\alpha$ setting $\omega_i = \Theta_i +\epsilon$ with $\epsilon \sim {\cal N}(0,1)$.
We let $Q_{\alpha,\Theta}$ denote the distribution on vectors defined by this generation process.  To apply the
PAC-Bayesian bound we will take $Q_{\alpha,0}$ as the prior distribution and $Q_{\alpha,\Theta}$ as the posterior distribution.
The PAC-Bayesian theorem then implies that for a dropout rate $\alpha$ selected before the draw of the sample
we have that with probability at least $1-\delta$ over the draw of the sample the following holds simultaneously for all $\Theta$.
$$L(Q_{\alpha,\Theta}) \leq \frac{1}{1-\frac{1}{2\lambda}}\parens{\hat{L}(Q_{\alpha,\Theta}) + \frac{\lambda \lmax}{N}\parens{\kl(Q_{\alpha,\Theta},Q_{\alpha,0}) + \ln \frac{1}{\delta}}}$$

To clarify formal notation we first consider the Boolean $d$-cube ${\cal B}$ which is the set of vector $s \in \mathbb{R}^d$
such that $s_i \in \{0,1\}$ for all $1 \leq i \leq d$.  We will call vectors $s \in{\cal B}$ ``sparsity patterns''.
We let $S_\alpha$ be the distribution on the $d$-cube ${\cal B}$ (the distribution on sparsity patterns) generated
by selecting each $s_i$ independently with the probability of $s_i = 0$ being $\alpha$.
For a sparsity pattern $s$ and for $\omega \in \mathbb{R}^d$ we will write $s \circ \omega$ for the Hadamard product defined by
$(s \circ \omega)_i = s_i\omega_i$.
We then have that a draw from $Q_{\alpha,\Theta}$ can be made by first drawing a sparsity pattern $s \sim S_\alpha$ and a noise vector
$\epsilon \sim {\cal N}(0,1)^d$ and then constructing $s \circ (\Theta + \epsilon)$.
More formally we have the following.
$$\expectsub{\omega \sim Q_{\alpha,\Theta}}{f(\omega)} = \expectsub{s \sim S_\alpha,\epsilon \sim {\cal N}(0,1)^d}{f(s\circ (\Theta + \epsilon))}$$
We then have 
\begin{eqnarray*}
\kl(Q_{\alpha,\Theta},\;Q_{\alpha,0}) & = & \expectsub{s \sim S_\alpha, \epsilon \sim {\cal N}(0,1)^d}
{\ln \frac{S_\alpha(s) e^{-\frac{1}{2}||s \circ \epsilon||^2}}{S_\alpha(s) e^{-\frac{1}{2}||s \circ (\Theta + \epsilon)||^2}}} \\
\\
& = & \expectsub{s \sim S_\alpha}{\frac{1}{2}||s \circ \Theta||^2} \\
\\
& = & \frac{1-\alpha}{2}||\Theta||^2
\end{eqnarray*}

The PAC-Bayesian bound then gives that, for a dropout rate $\alpha$ selected before the draw of the sample, with probability at least $1-\delta$ over the draw of the sample
the following holds simultaneously for all $\Theta$.
\begin{equation}
\label{SVB2}
L(Q_{\alpha,\Theta}) \leq \frac{1}{1-\frac{1}{2\lambda}}\parens{\hat{L}(Q_{\alpha,\Theta}) + \frac{\lambda \lmax}{N}\parens{\frac{1-\alpha}{2}||\Theta||^2 + \ln \frac{1}{\delta}}}
\end{equation}
Comparing (\ref{SVB2}) with (\ref{SVB}) we see that a dropout rate of $\alpha$ reduces the complexity cost by a factor of $1-\alpha$.
However, for $\alpha$ very small we expect $\hat{L}(Q_{\alpha,\Theta})$ to be large.
We can optimize the right hand side of this bound by stochastic gradient descent using the following.
$$\nabla_\Theta \; \hat{L}(Q_{\alpha,\Theta}) = \frac{1}{N} \sum_{i=1}^N \;\expectsub{s \sim S_\alpha,\;\epsilon\sim {\cal N}(0,1)^d}{\nabla_\Theta L(h_{s \circ (\Theta + \epsilon)},s_i)}$$

\subsection{The PAC-Bayesian Posterior}

It is important to note that (\ref{PAC-Bayes}) has a closed-form solution for the distribution $Q$ minimizing the bound.
$$Q^* = \argmin_Q\;\hat{L}(Q) + \frac{\lambda \lmax}{N} \kl(Q,P)$$
In the case where the rule space ${\cal H}$ is finite
we have the constraint that $\sum_{h \in {\cal H}}Q(h) = 1$ and a straightforward application of the KTT conditions yields the following.
\begin{eqnarray*}
Q^*(h) = Q_\lambda(h) & = & \frac{1}{Z_\lambda} P(h)e^{\frac{-N \hat{L}(h)}{\lambda\lmax}} \\
Z_\lambda & = & \expectsub{h \sim P}{e^{-\frac{N}{\lambda \lmax} \hat{L}(h)}}
\end{eqnarray*}
Here we can think of $Q_\lambda$ as ``the'' PAC-Bayesian posterior distribution for regularization parameter $\lambda$.  It is important to note
that the choice of $\lmax$ strongly influences the posterior distribution. In the case of (\ref{PAC-Bayes-Uniform}) we can optimize over $\lambda$ as follows.
$$Q^* = Q_{\lambda_{i^*}}\;\;\;\;\;\;\;\;i^* = \argmin_{1 \leq i \leq k} \; \hat{L}(Q_{\lambda_i}) + \frac{\lambda \lmax}{N}\kl(Q_{\lambda_i},P)$$
$$L(Q^*) \leq \min_{1 \leq i \leq k}\; \frac{1}{1-\frac{1}{2\lambda_i}}\parens{\hat{L}(Q_{\lambda_i}) + \frac{\lambda_i \lmax}{N}\parens{\kl(Q_{\lambda_i},P) + \ln \frac{k}{\delta}}}$$

\section{A Training-Variance Bound}
\label{sec:local}

We now consider a fixed learning algorithm ${\cal A}$ which takes as input a sample $S \sim D^N$
and returns a rule distribution $Q_{\cal A}(S)$.  For a given learning algorithm ${\cal A}$
we now consider the expected loss $\expectsub{S\sim D^N}{L(Q_{\cal A}(S))}$ and the expected posterior
$$\bar{Q}_{\cal A}(h) = \expectsub{S\sim D^N}{Q_{\cal A}(S)(h)}.$$
The training-variance bound is the following where
we will write $\expectsub{S}{f(S)}$ for $\expectsub{S \sim D^N}{f(S)}$.
\begin{theorem}
For any fixed learning algorithm ${\cal A}$ and for $\lambda > \frac{1}{2}$ we have
\begin{equation}
\label{expected2}
\mbox{\footnotesize $\expectsub{S}{L(Q_{\cal A}(S))} \leq \frac{1}{1-\frac{1}{2\lambda}}\;\parens{\expectsub{S}{\hat{L}(Q_{\cal A}(S))} + \frac{\lambda\lmax}{N}\expectsub{S }{\kl(Q_{\cal A}(S),\bar{Q}_{\cal A})}}$.}
\end{equation}
\end{theorem}
We can think of $\expectsub{S}{\kl(Q_{\cal A}(S),\bar{Q}_{\cal A})}$ as a measure of the variation of $Q_{\cal A}(S)$ over the draw of $S$.
For the PAC-Bayesian bound (\ref{PAC-Bayes}) we have a closed form solution for the optimal posterior.  But for the training-variance bound (\ref{expected2})
we do not have a solution for the optimal algorithm.  The training-variance bound seems to motivate variance-reduction methods
such as bagging \cite{breiman1996bagging}.

The training-variance bound is an immediate corollary of the following more general theorem
which is implicit in Catoni~\cite{catoni2007pac} and which is proved in appendix~\ref{sec:expected}.
\begin{theorem}
\label{thm:expected}
For any rule distribution $P$, learning algorithm ${\cal A}$, and for $\lambda > \frac{1}{2}$, we have
\begin{equation}
\label{expected}
\mbox{\footnotesize $\expectsub{S}{L(Q_{\cal A}(S))} \leq \frac{1}{1-\frac{1}{2\lambda}}\;\parens{\expectsub{S}{\hat{L}(Q_{\cal A}(S))} + \frac{\lambda\lmax}{N}\expectsub{S }{\kl(Q_{\cal A}(S),P)}}$.}
\end{equation}
\end{theorem}
It was observed by Langford \cite{langford1999microchoice}
that the rule distribution $P$ minimizing $\expectsub{S}{\kl(Q_{\cal A}(S),P)}$ is $\bar{Q}_{\cal A}$.
This can be shown as follows.
\begin{eqnarray*}
\expectsub{S }{\kl(Q_{\cal A}(S),P)} & = & \expectsub{S,\;h \sim Q_{\cal A}(S)}{\ln \frac{Q_{\cal A}(S)(h)}{P(h)}}  \\
\\
 & = & \expectsub{S,\;h \sim Q_{\cal A}(S)}{\ln \frac{Q_{\cal A}(S)(h)}{\bar{Q}_{\cal A}(h)}} + \expectsub{h \sim \bar{Q}_{\cal A}}{\ln \frac{\bar{Q}(h)}{P(h)}} \\
\\
& = & \expectsub{S}{\kl(Q_{\cal A}(S),\bar{Q}_{\cal A})} + \kl(\bar{Q}_{\cal A},P)
\end{eqnarray*}
This shows that (\ref{expected2}) dominates (\ref{expected}) and is much better when
$\kl(\bar{Q}_{\cal A},P)$ is large.

For a given learning algorithm ${\cal A}$ we can insert
$\bar{Q}_{\cal A}$ for $P$ in the PAC-Bayesian bound (\ref{PAC-Bayes}) yielding the following high confidence version of (\ref{expected2}).
\begin{theorem}
For any given learning algorithm ${\cal A}$ and $\lambda > \frac{1}{2}$ we have the following with probability at least $1-\delta$ over the draw of the sample.
\begin{equation}
\label{local}
\mbox{\small $L(Q_{\cal A}(S)) \leq \frac{1}{1-\frac{1}{2\lambda}}\parens{\hat{L}(Q_{\cal A}(S)) + \frac{\lambda\lmax}{N}\parens{\kl(Q_{\cal A}(S),\bar{Q}_{\cal A}) + \ln \frac{1}{\delta}}}$}
\end{equation}
\end{theorem}

\section{Applying the Training-Variance Bound to the PAC-Bayesian Posterior}

We now consider the learning algorithm that maps a sample $S$ to the PAC-Bayesian posterior $Q_\lambda(S)$.
Here $\lambda$ is a parameter of the learning algorithm.
We should note that, although $Q_\lambda$ is the posterior optimizing the PAC-Bayesian bound~(\ref{PAC-Bayes}), it seems unlikely
that $Q_\lambda$ is the algorithm optimizing the training-variance bound (\ref{expected2}).  Also, as we will see below, the analysis given here is somewhat
loose.

Following Catoni \cite{catoni2007pac} and Lever et al. \cite{lever2010distribution} we
approximate $\bar{Q}_\lambda$ with the following.
\begin{eqnarray*}
\ddot{Q}_{\lambda}(h) & = & \frac{1}{\ddot{Z}_\lambda}P(h)e^{\frac{-N L(h)}{\lambda\lmax}} \\
\\
\ddot{Z}_\lambda & = &\expectsub{h \sim P}{e^{\frac{-N L(h)}{\lambda\lmax}}}
\end{eqnarray*}
Inserting $\ddot{Q}_\lambda$ for $P$ in (\ref{expected}) gives
that for $\gamma > \frac{1}{2}$ we have the following.
\begin{equation}
\label{intermediate}
\mbox{\footnotesize $\expectsub{S}{L(Q_\lambda(S))} \leq \frac{1}{1-\frac{1}{2\gamma}}\parens{\expectsub{S}{\hat{L}(Q_\lambda(S))} + \frac{\gamma\lmax}{N}\expectsub{S}{\kl(Q_\lambda(S),\ddot{Q}_\lambda)}}$}
\end{equation}
Note that we allow $\gamma$ to be different from $\lambda$.  We now have the following bound from Catoni \cite{catoni2007pac} and whose proof is given in appendix~\ref{sec:applying}.
\begin{equation}
\label{BoundOnKL}
\expectsub{S}{\kl(Q_\lambda(S),\ddot{Q}_\lambda)} \leq \frac{N}{\lambda\lmax}\parens{\expectsub{S}{L(Q_\lambda(S))} - \expectsub{S}{\hat{L}(Q_\lambda(S))}}
\end{equation}
By inserting (\ref{BoundOnKL}) into (\ref{intermediate}), setting $\gamma = \frac{1}{2}\lambda$ and solving for $\expectsub{S}{L(Q_\lambda(S))}$ one can derive
the following for $\lambda > 2$.
\begin{equation}
\label{localposterior}
\mbox{\footnotesize $\expectsub{S}{L(Q_\lambda(S))} \leq \frac{1}{1-\frac{2}{\lambda}}\expectsub{S}{\hat{L}(Q_\lambda(S))}$}
\end{equation}
Note that $\frac{1}{2\lambda}$ in the Occam and PAC-Bayesian bound has been replaced with $\frac{2}{\lambda}$.  Also note that $\lmax$ appears in the definition of $Q_\lambda(S)$.

To get a corresponding high-confidence bound we first note that by inserting $\ddot{Q}_\lambda$ for $P$ in the PAC-Bayesian bound (\ref{PAC-Bayes}) we get that, for $\gamma > \frac{1}{2}$,
with probability at least $1-\delta$ over the draw of the sample we have
\begin{equation}
\label{local2}
\mbox{\small $L(Q_\lambda(S)) \leq \frac{1}{1-\frac{1}{2\gamma}}\parens{\hat{L}(Q_\lambda(S)) + \frac{\gamma\lmax}{N}\parens{\kl(Q_\lambda(S),\ddot{Q}_\lambda) + \ln \frac{1}{\delta}}}$.}
\end{equation}
This can be combined with the following whose proof is given in appendix~\ref{sec:applying}.
\begin{lemma}
For $\lambda > 0$, with probability at least $1-\delta$ over the draw of the sample we have
\begin{equation}
\label{BoundOnKL2}
\kl(Q_\lambda(S),\ddot{Q}_\lambda) \leq \frac{N}{\lambda \lmax}\parens{L(Q_\lambda(S)) - \hat{L}(Q_\lambda(S))} + \frac{N}{\lambda}\sqrt{\frac{\ln\frac{1}{\delta}}{2N}}.
\end{equation}
\end{lemma}
Taking a union bound over (\ref{BoundOnKL2}) and (\ref{local2}) so that both are true simultaneously,
then inserting~(\ref{BoundOnKL2}) into (\ref{local2}), setting $\gamma = \frac{1}{2}\lambda$, and solving for $L(Q_\lambda(S))$,
yields that with probability at least $1-\delta$ over the draw of the sample we have
\begin{equation}
\label{secondloc}
L(Q_\lambda(S))  \leq  \frac{1}{1-\frac{2}{\lambda}}\parens{\hat{L}(Q_\lambda(S)) + \lmax\sqrt{\frac{\ln\frac{2}{\delta}}{2N}} + \frac{\lambda \lmax\ln\frac{2}{\delta}}{N}}
\end{equation}

Improvements in these bounds should be possible.
To see this consider the PAC-Bayesian bound (\ref{PAC-Bayes}) for which $Q_\lambda(S)$ is the optimal posterior.
$$L(Q_\lambda(S)) \leq \frac{1}{1 - \frac{1}{2\lambda}}\parens{\hat{L}(Q_\lambda(S)) + \frac{\lambda \lmax}{N}\parens{\kl(Q_\lambda(S),P) + \ln \frac{1}{\delta}}}$$
Replacing $P$ by $\bar{Q}_\lambda$ should significantly improve this bound.  However, replacing $P$ by $\ddot{Q}_\lambda$ and then
inserting (\ref{BoundOnKL2}) makes the bound vacuous.

\section{Incorporating Empirical Loss Variance}

For a given rule $h$ and sample $S = \{s_1,\ldots,s_n\}$ one can measure an empirical loss variance.
$$\hat{\sigma}^2(h) = \frac{1}{N-1} \; \sum_{i=1}^N (L(h,s_i) - \hat{L}(h))^2$$
It is natural to ask whether tighter bounds are possible if we allow the bounds to involve $\hat{\sigma}^2(h)$.
Audibert, Munos and Szepesvari \cite{audibert2006use} give a bound motivated by this question.
Consider a random variable $x \in [0,\lmax]$ with expectation $\mu$
and an IID sample $\{x_1,\ldots,x_n\}$ with empirical mean $\hat{\mu}$ and empirical variance $\hat{\sigma}^2$.
Audibert, Munos and Szepesvari prove that the following holds with probability at least $1-\delta$ over the draw of the sample.
$$\mu \leq \hat{\mu} + \sqrt{\frac{2\hat{\sigma}^2\ln\frac{3}{\delta}}{N}} + \frac{3\lmax\ln\frac{3}{\delta}}{N}$$
Taking a union bound over a prior $P$ we get that with probability at least $1-\delta$ the following holds
for all $h \in {\cal H}$.
$$L(h) \leq \hat{L}(h) + \sqrt{\frac{2\hat{\sigma}^2(h)\parens{\ln \frac{1}{P(h)} + \ln\frac{3}{\delta}}}{N}} + \frac{3\lmax\parens{\ln \frac{1}{P(h)} + \ln\frac{3}{\delta}}}{N}$$
To show the limitations of these bounds we consider the best possible case where $\hat{\sigma}^2(h) = 0$.  For this case we have the following theorem whose proof is given in appendix~\ref{sec:attempt1}.
\begin{theorem}
\label{thm:attempt1}
With probability at least $1-\delta$ over the draw of the sample we
have that the following holds for all $h$ such that $\hat{\sigma}^2(h) = 0$.
\begin{equation}
\label{OccamLower}
L(h) \leq \hat{L}(h) + \frac{\lmax\parens{\ln \frac{1}{P(h)} + \ln \frac{1}{\delta}}}{N-1}
\end{equation}
\end{theorem}
The inequality (\ref{OccamLower}) is essentially the best that can be done using a union bound over $P(h)$.
The basic idea is that even if $\hat{\sigma}^2(h) = 0$ one cannot rule out the possibility of outliers which happened not to occur
in the data.  The probability of outliers cannot be bounded to be less than $(\ln \frac{1}{P(h)} + \ln \frac{1}{\delta})/N$ and an outlier can have loss $\lmax$
so (\ref{OccamLower}) is the best that can be done.

But (\ref{OccamLower}) is not significantly tighter than the general Occam bound (\ref{Occam}).
In particular, by taking $\lambda = 1$ in (\ref{Occam}) we get a bound that is only a factor of 2 worse than (\ref{OccamLower}).
If (\ref{OccamLower}) is dominated by $\hat{L}(h)$ then we can take $\lambda$ in the Occam bound (\ref{Occam}) to be large and the two bounds are essentially the same.

\section{Conclusion}

This paper focuses on three generalization bounds --- an Occam bound, a PAC-Bayesian bound, and a training-variance bound.  The Occam bound and PAC-Bayesian bound
seem to be important primarily because they provide the conceptual foundation required for the proof of the training-variance bound which dominates the other two.
While the PAC-Bayesian posterior defines a learning algorithm optimizing the PAC-Bayesian bound, there is no known analogous optimal algorithm for the  training-variance bound.
The bound seems to suggest variance reduction methods such as boosting.  There is clearly room for improved theoretical understanding of the consequences of the
training-variance bound.

\bibliographystyle{plain}
\bibliography{Catoni}

\begin{thebibliography}{10}

\bibitem{angluin1977fast}
Dana Angluin and Leslie~G Valiant.
\newblock Fast probabilistic algorithms for hamiltonian circuits and matchings.
\newblock In {\em Proceedings of the ninth annual ACM symposium on Theory of
  computing}, pages 30--41. ACM, 1977.

\bibitem{audibert2006use}
Jean-Yves Audibert, R{\'e}mi Munos, Csaba Szepesvari, et~al.
\newblock Use of variance estimation in the multi-armed bandit problem.
\newblock 2006.

\bibitem{breiman1996bagging}
Leo Breiman.
\newblock Bagging predictors.
\newblock {\em Machine learning}, 24(2):123--140, 1996.

\bibitem{catoni2007pac}
Olivier Catoni.
\newblock Pac-bayesian supervised classification: the thermodynamics of
  statistical learning.
\newblock {\em arXiv preprint arXiv:0712.0248}, 2007.

\bibitem{deng2013new}
Li~Deng, Geoffrey Hinton, and Brian Kingsbury.
\newblock New types of deep neural network learning for speech recognition and
  related applications: An overview.
\newblock {\em Proc. ICASSP}, 2013.

\bibitem{germain2009pac}
Pascal Germain, Alexandre Lacasse, Fran{\c{c}}ois Laviolette, and Mario
  Marchand.
\newblock Pac-bayesian learning of linear classifiers.
\newblock In {\em Proceedings of the 26th Annual International Conference on
  Machine Learning}, pages 353--360. ACM, 2009.

\bibitem{langford2006}
John Langford.
\newblock Tutorial on practical prediction theory for classification.
\newblock {\em Journal of Machine Learning Research}, 6(1):273, 2006.

\bibitem{langford1999microchoice}
John Langford and Avrim Blum.
\newblock Microchoice bounds and self bounding learning algorithms.
\newblock In {\em Proceedings of the twelfth annual conference on Computational
  learning theory}, pages 209--214. ACM, 1999.

\bibitem{lever2010distribution}
Guy Lever, Fran{\c{c}}ois Laviolette, and John Shawe-Taylor.
\newblock Distribution-dependent pac-bayes priors.
\newblock In {\em Algorithmic Learning Theory}, pages 119--133. Springer, 2010.

\bibitem{maurer2004}
Andreas Maurer.
\newblock A note on the pac-bayesian theorem.
\newblock {\em arXiv preprint cs/0411099}, 2004.

\bibitem{McAllester99}
David~A. McAllester.
\newblock Pac-bayesian model averaging.
\newblock In {\em COLT}, pages 164--170, 1999.

\bibitem{seeger2003}
Matthias Seeger.
\newblock Pac-bayesian generalisation error bounds for gaussian process
  classification.
\newblock {\em J. Mach. Learn. Res}, 3(2):233--269, 2003.

\bibitem{PAC}
Lelsie Valiant.
\newblock A theory of the learnable.
\newblock {\em Communications of the ACM}, 27, 1984.

\end{thebibliography}

\appendix

\section{Proof of Theorem~\ref{thm:PAC-Bayes}}
\label{sec:PAC-Bayes-Proof}

All proofs in these appendices are adapted from Catoni \cite{catoni2007pac} except for the proof of theorem~\ref{thm:attempt1}
which is straightforward.

\medskip
The theorem states that for $\lambda > \frac{1}{2}$ selected before the draw of the sample (for any fixed $\lambda > 1/2$) we have that, with probability
at least $1-\delta$ over the draw of the sample, the following holds simultaneously for all distributions $Q$ on ${\cal H}$.
$$L(Q) \leq \frac{1}{1-\frac{1}{2\lambda}}\parens{\hat{L}(Q) + \frac{\lambda \lmax}{N}\parens{\kl(Q,P) + \ln \frac{1}{\delta}}}$$

We will consider the case of $\lmax = 1$, the general case follows by rescaling the loss function.
For real numbers $p,q \in [0,1]$ we define $\kl(q,p)$ to be the divergence from a
Bernoulli variable with bias $q$ to a Bernoulli variable with bias $p$.
$$\kl(q,p) = q\ln\frac{q}{p} + (1-q)\ln \frac{1-p}{1-q}$$
For a real number $\gamma$ we define
$$\kl_\gamma(q,p) = \gamma q - \ln\parens{1-p + p e^\gamma}.$$
By a straightforward optimization over $\gamma$ one can show
\begin{equation}
\label{kl}
\kl(q,p) = \sup_\gamma \kl_\gamma(q,p).
\end{equation}
Now consider a random variable $x$ with $x \in [0,1]$ and with mean $\mu$.  Let $\hat{\mu}$ be the mean of $N$ independent draws of $x$.
We first show that for any fixed $\gamma$ we have
\begin{equation}
\label{iidbasic}
\expectsub{}{e^{N\kl_\gamma(\hat{\mu},\mu)}} \leq 1.
\end{equation}
To see this note that $\expectsub{}{e^{N\gamma \hat{\mu}}} = \parens{\expectsub{}{e^{\gamma x}}}^N$.
For $x \in[0,1]$ we note that the convexity of the exponential function implies
$e^{\gamma x} \leq 1 - x + xe^{\gamma}$.  This gives $\expectsub{}{e^{N\gamma \hat{\mu}}} \leq (1-\mu + \mu e^{\gamma})^N$.
Dividing by the right hand side
gives $\expectsub{}{e^{N\parens{\gamma \hat{\mu} - \ln(1-\mu + \mu e^{\gamma})}}} \leq 1$ which is the same as (\ref{iidbasic}).
It is interesting to note that
\begin{eqnarray*}
\expectsub{S \sim D^N}{e^{N\kl(\hat{\mu},\mu)}} & = & \expectsub{S \sim D^N}{\sup_\gamma\;e^{N\kl_\gamma(\hat{\mu},\mu)}} \\
\\
& \geq & \sup_\gamma\; \expectsub{S \sim D^N}{e^{N\kl_\gamma(\hat{\mu},\mu)}}  \leq 1
\end{eqnarray*}
For $h$ fixed (\ref{iidbasic}) implies the following.
\begin{eqnarray}
\expectsub{S\sim D^N}{e^{N\kl_\gamma(\hat{L}(h),L(h))}} & \leq & 1 \nonumber \\
\expectsub{h \sim P}{\expectsub{S\sim D^N}{e^{N \kl_\gamma(\hat{L}(h),L(h))}}} & \leq & 1\nonumber \\
\label{departure2}
\expectsub{S\sim D^N}{\expectsub{h \sim P}{e^{N \kl_\gamma(\hat{L}(h),L(h))}}} & \leq & 1
\end{eqnarray}
Applying Markov's inequality to (\ref{departure2}) we get that with probability at least $1-\delta$ over the draw of $S$ we have
\begin{equation}
\label{Markov}
\expectsub{h \sim P}{e^{N\kl_\gamma(\hat{L}(h),L(h,P))}} \leq \frac{1}{\delta}.
\end{equation}
Next we observe the shift of measure lemma
\begin{equation}
\label{shift}
\expectsub{h \sim Q}{f(h)} \leq \kl(Q,P) + \ln \expectsub{h\sim P}{e^{f(h)}}
\end{equation}
which can be derived as follows.
\begin{eqnarray*}
\expectsub{h \sim Q}{f(h)} & = & \expectsub{h\sim Q}{\ln e^{f(h)}} \\
			& = & \expectsub{h\sim Q}{\ln \frac{P(h)}{Q(h)} e^{f(h)} + \ln\frac{Q(h)}{P(h)}} \\
			& \leq &  \ln \expectsub{h\sim Q}{\frac{P(h)}{Q(h)} e^{f(h)}} + \kl(Q,P)\\
			& = & \kl(Q,P) + \ln \expectsub{h\sim P}{e^{f(h)}}
\end{eqnarray*}
Setting $f(h) = N\kl_\gamma(\hat{L}(h),L(h))$ in (\ref{shift}) and using (\ref{Markov}) we get
$$\expectsub{h \sim Q}{N\kl_\gamma(\hat{L}(h),L(h))} \leq \kl(Q,P) + \ln\frac{1}{\delta}.$$
Noting that $\kl_\gamma(q,p)$ is jointly convex in $q$ and $p$ we get
\begin{equation}
\label{pre-PAC-Bayes}
\kl_\gamma(\hat{L}(Q),L(Q)) \leq \frac{1}{N}\parens{\kl(Q,P) + \ln\frac{1}{\delta}}.
\end{equation}
Theorem \ref{thm:PAC-Bayes} is now implied by the following lemma.

\medskip

\begin{lemma}
\label{lem:invert}
For $\lambda > \frac{1}{2}$, if $D_{-\frac{1}{\lambda}}(p,q) \leq c$ then $p \leq \frac{1}{1-\frac{1}{2\lambda}}\parens{q + \lambda c}$.
\end{lemma}

\begin{proof}
Let $\gamma$ abbreviate $-\frac{1}{\lambda}$.  We are given $q\gamma - \ln\parens{1-p + p e^\gamma} \leq c$.
Since $\lambda > \frac{1}{2}$ we have $\gamma \in (-2,0)$. We then get
$$p \leq \frac{1-e^{\gamma q - c}}{1 - e^\gamma}.$$
Applying $e^\gamma \geq 1+\gamma$ in the numerator and $e^\gamma \leq 1 + \gamma + \frac{1}{2}\gamma^2 \leq 1$ for $\gamma \in (-2,0)$ in the denominator
we get
$$p \leq \frac{-\gamma q + c}{-\gamma - \frac{1}{2}\gamma^2} = \frac{q - \frac{c}{\gamma}}{1 + \frac{1}{2}\gamma}$$
Replacing $\gamma$ by $-1/\lambda$ proves the lemma.
\end{proof}

\section{Proof of Theorem~\ref{thm:expected}}
\label{sec:expected}

The theorem states that for distribution $P$ on rules, any algorithm ${\cal A}$, and for $\lambda > \frac{1}{2}$, we have
$$\mbox{\footnotesize $\expectsub{S}{L(Q_{\cal A}(S))} \leq \frac{1}{1-\frac{1}{2\lambda}}\;\parens{\expectsub{S}{\hat{L}(Q_{\cal A}(S))} + \frac{\lambda\lmax}{N}\expectsub{S }{\kl(Q_{\cal A}(S),P)}}$.}$$

\begin{proof}
The proof is a slight modification of the proof of theorem~\ref{thm:PAC-Bayes} given in section~\ref{sec:PAC-Bayes-Proof}.
By the shift of measure lemma (\ref{shift}) we have the following for any fixed sample $S$.
$$\expectsub{h \sim Q_{\cal A}(S)}{N\kl_\gamma(\hat{L}(h),L(h))} \leq \kl(Q_{\cal A}(S),P) +  \ln \expectsub{h \sim P}{e^{N\kl_\gamma(\hat{L}(h),L(h))}}.$$
By the joint convexity of $\kl_\gamma$ we then have
$$\kl_\gamma(\hat{L}(Q_{\cal A}(S)),L(Q_{\cal A}(S))) \leq \frac{1}{N}\parens{\kl(Q_{\cal A}(S),P) +  \ln \expectsub{h \sim P}{e^{N\kl_\gamma(\hat{L}(h),L(h))}}}.$$
Taking the expectation of both sides with respect to $S$ and using the convexity of $\kl_\gamma$ and the concavity of $\ln$ we get
\begin{eqnarray*}
 & & \kl_\gamma\parens{\expectsub{S}{\hat{L}(Q_{\cal A}(S))},\expectsub{S}{L(Q_{\cal A}(S))}} \\
& \leq & \frac{1}{N}\parens{\expectsub{S}{\kl(Q_{\cal A}(S),P)} +  \ln \expectsub{h \sim P,S\sim D^N}{e^{N\kl_\gamma(\hat{L}(h),L(h))}}}.
\end{eqnarray*}
Theorem~\ref{thm:expected} now follows from (\ref{iidbasic}) and lemma~\ref{lem:invert}.
\end{proof}

\section{Proof of (\ref{BoundOnKL}) and (\ref{BoundOnKL2})}
\label{sec:applying}

(\ref{BoundOnKL}) is the following.
$$\expectsub{S}{\kl(Q_\lambda(S),\ddot{Q}_\lambda)} \leq \frac{N}{\lambda\lmax}\parens{\expectsub{S}{L(Q_\lambda(S))} - \expectsub{S}{\hat{L}(Q_\lambda(S))}}$$

\begin{proof}
\begin{eqnarray*}
\expectsub{S}{\kl(Q_\lambda(S),\ddot{Q}_\lambda)} & = & \expectsub{S,h \sim Q_\lambda(S)}{ \ln\frac{Q_\lambda(S)(h)}{\ddot{Q}_\lambda(h)}} \\
\\
 & = & \expectsub{S,h \sim Q_\lambda(S)}{\frac{N}{\lambda\lmax} L(h) - \frac{N}{\lambda\lmax} \hat{L}(h))}  \\
& & - \expectsub{S}{\ln Z_\lambda(S)} + \ln\ddot{Z}_\lambda  \\
 \\
 & = & \frac{N}{\lambda \lmax}\parens{\expectsub{S}{L(Q_\lambda(S))} - \expectsub{S}{\hat{L}(Q_\lambda(S))}} \\
& & - \expectsub{S}{\ln Z_\lambda(S)} + \ln\ddot{Z}_\lambda
\end{eqnarray*}
But the log partition function is convex in energy which gives
\begin{eqnarray*}
\expectsub{S}{\ln Z_\lambda(S)} & = & \expectsub{S}{\ln \expectsub{h \sim P}{e^{-\frac{N}{\lambda\lmax} \hat{L}(h)}}}  \\
& \geq & \ln \expectsub{h \sim P}{e^{-\frac{N}{\lambda\lmax}\expectsub{S}{\hat{L}(h)}}} \\
\\
& = & \ln \expectsub{h \sim P}{e^{-\frac{N}{\lambda\lmax}L(h)}} \\
\\ & = & \ddot{Z}_\lambda
\end{eqnarray*}
\end{proof}

\noindent (\ref{BoundOnKL2}) states that with probability at least $1-\delta$ we have

$$\kl(Q_\lambda(S),\ddot{Q}_\lambda) \leq \frac{N}{\lambda \lmax}\parens{\expectsub{S}{L(Q_\lambda(S)) - \hat{L}(Q_\lambda(S))}} + \frac{N}{\lambda}\sqrt{\frac{\ln\frac{1}{\delta}}{2N}}.$$

\begin{proof}
\begin{eqnarray*}
\kl(Q_\lambda(S),\ddot{Q}_\lambda) & = & \expectsub{h \sim Q_\lambda(S)}{ \ln\frac{Q_\lambda(S)(h)}{\ddot{Q}_\lambda(h)}}  \\
\\
 & = & \expectsub{h \sim Q_\lambda(S)}{\frac{N}{\lambda\lmax} L(h) - \frac{N}{\lambda\lmax} \hat{L}(h))}  \\
& & - \ln Z_\lambda(S) + \ln\ddot{Z}_\lambda  \\
 \\
 & = & \frac{N}{\lambda\lmax}\parens{L(Q_\lambda(S)) -\hat{L}(Q_\lambda(S)))}  \\
& & - \ln Z_\lambda(S) + \ln\ddot{Z}_\lambda
\end{eqnarray*}
\begin{eqnarray*}
\ln Z_\lambda(S) & = & \ln \expectsub{h \sim P}{e^{-\frac{N}{\lambda\lmax} \hat{L}(h)}}  \\
 & = & \ln \expectsub{h \sim \ddot{Q}_\lambda}{\frac{P(h)}{\ddot{Q}_\lambda(h)} e^{-\frac{N}{\lambda\lmax} \hat{L}(h)}} 
\end{eqnarray*}
\begin{eqnarray*}
  & \ge & \expectsub{h \sim \ddot{Q}_\lambda}{\ln \frac{P(h)}{\ddot{Q}_\lambda(h)}-\frac{N}{\lambda\lmax} \hat{L}(h)}  \\
 \\
 & = & \ln \ddot{Z}_\lambda + \frac{N(\; L(\ddot{Q}_\lambda) - \hat{L}(\ddot{Q}_\lambda)\;)}{\lambda\lmax}
\end{eqnarray*}
Since $\lambda$ is selected before the draw of the sample, a Hoeffding bound can be used to bound $L(\ddot{Q}_\lambda) - \hat{L}(\ddot{Q}_\lambda)$ yielding
that with probability at least $1-\delta$ over the draw of the sample we have
\begin{equation}
\label{partition2}
\ln Z_\lambda(S) \geq  \ln \ddot{Z}_\lambda - \frac{N}{\lambda}\sqrt{\frac{\ln \frac{1}{\delta}}{2N}}.
\end{equation}
\end{proof}

\section{Proof of Theorem~\ref{thm:attempt1}}
\label{sec:attempt1}

The theorem states that
with probability at least $1-\delta$ over the draw of the sample we
we have that the following holds for all $h$ such that $\hat{\sigma}^2(h) = 0$.
$$L(h) \leq \hat{L}(h) + \frac{\lmax\parens{\ln \frac{1}{P(h)} + \ln \frac{1}{\delta}}}{N-1}$$

\begin{proof}
Consider a sample $S =\{s_1,\ldots,s_n\}$.  We let the sample $s_1$ define a target loss value for each rule.  We consider a sample $s_i$ for $i>1$ to be an ``outlier''
for rule $h$ if $L(h,s_i) \not = L(h,s_1)$.  We can then use the standard ``realizable'' analysis over the sample $\{s_2,\ldots,s_N\}$ to bound the outlier rate.
More specifically, let $\mu(h)$ be the probability that a new draw of a situation from $D$ is an outlier for $h$.
$$\mu(h) = P_{s \sim D}(L(h,s) \not = L(h,s_1))$$
We will first show that with probability at least $1-\delta$ over the draw of
$\{s_2,\ldots,s_N\}$ we have that the following holds simultaneously for all $h$ such that $\hat{\sigma}^2(h) = 0$.
\begin{equation}
\label{internal}
\mu(h) \leq \frac{\ln \frac{1}{P(h)} + \ln \frac{1}{\delta}}{N-1}
\end{equation}
The probability over the draw of $\{s_2,\ldots,s_N\} \sim D^{N-1}$
that $\hat{\sigma}^2(h) = 0$  equals $(1-\mu(h))^{N-1} \leq e^{-(N-1)\mu(h)}$. So if $h$ violates (\ref{internal}) then
the probability that $\sigma^2(h) = 0$ is at most $P(h)\delta$.  By the union bound the probability that there exists an $h$
with $\hat{\sigma}^2(h) = 0$ and violating (\ref{internal}) is at most $\sum_h P(h)\delta = \delta$ and thus with high probability (\ref{internal}) holds
for all $h$.
The theorem then follows from the observation that if $\hat{\sigma}^2(h) = 0$ then $\hat{L}(h) = L(h,s_1)$ and
$L(h) \leq L(h,s_1) + \lmax \mu(h)$.
\end{proof}

\end{document}